\documentclass[article]{amsart}  

\usepackage{Nyquist_sty}
\usepackage{xcolor}

\usepackage{hyperref}

\numberwithin{equation}{section}

\newcommand{\Lip}[1]{\mathrm{Lip}(#1)}

\begin{document}   

\title[LDP for entropy regularised IPS in zero-sum games]{Large deviations for interacting particle dynamics for finding mixed equilibria in zero-sum games}


\author[V.~Nilsson]{Viktor Nilsson}
\author[P.~Nyquist]{Pierre Nyquist}
\address[V.~Nilsson]{KTH Royal Institute of Technology}
\address[P.~Nyquist]{Chalmers University of Technology and University of Gothenburg}
\email{{vikn@kth.se}, {pnyquist@chalmers.se}}
\thanks{}






\subjclass[2010]{}

\keywords{Large deviations; minmax games; Mixed Nash equilibrium; Generative adversarial networks; Interacting particle system}


\begin{abstract} 
Finding equilibrium points in continuous minmax games has become a key problem within machine learning, in part due to its connection to the training of generative adversarial networks and reinforcement learning. Because of existence and robustness issues, recent developments have shifted from pure equilibria to focusing on mixed equilibrium points. 
In this work we consider a method for finding mixed equilibria in two-layer zero-sum games based on entropic regularisation, where the two competing strategies are represented by two sets of interacting particles. We show that the sequence of empirical measures of the particle system satisfies a large deviation principle as the number of particles grows to infinity, and how this implies convergence of the empirical measure and the associated Nikaid\^o-Isoda error, complementing existing law of large numbers results.

\end{abstract}   

\maketitle

\section{Introduction}
\label{sec:intro}
Zero-sum games are at the heart of game theory and an important concept in a range of areas of applied mathematics. The development of efficient methods for finding equilibria in such games is therefore a central problem in algorithmic game theory \cite{BO98}. The problem is often cast as finding an optimum of a sequential decision problem, and much of the literature in mathematical programming and algorithmic game theory focuses on the setting of convex-concave objective functions. 

In the context of machine learning, zero-sum minmax games have become a central concept, in part due to their connection to generative adversarial networks (GANs) \cite{goodfellow2014generative}, adversarial training \cite{Madry2018}, and reinforcement learning \cite{Busoniu2008}. Because of this, there is a renewed interest in the study of algorithms for finding equilibria in continuous games without convexity-concavity assumptions. As an example, Goodfellow et al.\ \cite{goodfellow2014generative} originally cast the GAN training problem as follows: find a Nash equilibrium of the zero-sum game
\begin{align*}
    \min _{G} \max _{D} V(D, G)=\mathbb{E}_{\boldsymbol{x} \sim p_{\text {data }}(\boldsymbol{x})}[\log D(\boldsymbol{x})]+\mathbb{E}_{\boldsymbol{z} \sim p_{\boldsymbol{z}}(\boldsymbol{z})}[\log (1-D(G(\boldsymbol{z})))].
\end{align*}
Models of this type are notoriously difficult to train \cite{salimans2016improved, arjovsky2017towards, arjovsky2017wasserstein,barnett2018convergence}, while also having enormous empirical success in a number of areas \cite{karras2021alias, engel2019gansynth, kim2020learning, yu2018generative, zhang2017stackgan, clark2019adversarial}. This has sparked an interest in developing computational methods for finding equilibria in general minmax games, as well as a surge in work on the theoretical underpinnings of GAN training; see, e.g., \cite{HLC19, DEJ+20, CKZ22, CG20, GM21, SAK21} and references therein.

In the absence of a convexity-concavity assumption, finding (pure) equilibria points in minmax games is a challenging problem \cite{DSZ21}. Much of the existing work on computational methods concerns algorithms that come with some type of local convergence. For work along these lines, with a focus on the machine learning context, see e.g. \cite{Heusel17, Daskalakis18, Balduzzi18, MLZ+18, Adolphs19, MR19, MJS19, RCJ19, JNJ20, Fiez21, Tsaknakis21}.

Motivated by the need for efficient methods for training GANs, \cite{HLC19} shift the focus from pure to mixed strategies and reformulates the zero-sum game that defines the training as a multi-agent optimisation problem in the space of probability measures. 
Building on this work, in \cite{DEJ+20} the authors analyse general minmax games, including GANs, and associated mixed Nash equilibria by connecting them to Langevin dynamics. They construct implementable training dynamics, and test them numerically, and provide a theoretical analysis based on gradient flows in the space of probability measures \cite{AGS08}. In \cite{Lu2023}, Lu studies a version of this dynamics, with an additional time-scale separation between the two types of particles (i.e., player dynamics). In the very recent work \cite{LuMon2025}, Lu and Monmarch\'e consider the dynamics from \cite{DEJ+20} in the setting of zero-sum games with $K \geq 2$ players. These works are all related to the open questions posed in \cite{WangChizat2024} related to mean-field Langevin dynamics in zero-sum games; see also the references in \cite{Lu2023, LuMon2025} for more on mean-field-type dynamics in the machine learning and zero-sum games intersection.

When using a system of interacting particles, as in \cite{HLC19, DEJ+20}, it is crucial to understand the properties of the system as the number of particles grows. Along with convergence of the particle system, e.g., law of large numbers and central limit results, an important aspect is to understand the fluctuations from such limits. In this note we therefore consider the training dynamics of \cite{DEJ+20} from a large deviations perspective. Under mild regularity conditions on the payoff function of the underlying game, we establish a large deviation principle (LDP) for the empirical measure of an interacting particle system corresponding to the Langevin Descent-Ascent dynamics of \cite{DEJ+20}. We show how slightly different versions of the convergence results of \cite{DEJ+20}, for the interacting particle systems and the associated Nikaid\^o-Isoda error, follow from this LDP. To prove the LDP, we use that the particle dynamics corresponding to the strategies of the two players can be formulated in a way that fits with the large deviation results of \cite{BDF12}. 
In addition to obtaining the results in this note, building on the rather general results of \cite{BDF12} also sets the stage for future work on dynamics with more general diffusion terms. We also posit that large deviation results of the type considered here can be a different type of starting point for addressing the open problems in \cite{WangChizat2024}.

Large deviations have proven to be extremely useful in the analysis and design of Monte Carlo methods, in particular in the rare-event context; for some examples see \cite{AG07, BD19, DupuisWangSubsol, DSZ15, DDN18, BNS21}. Whereas tools from large deviations theory, and their connections to stochastic control, have been used extensively in the Monte Carlo setting, they are largely unexplored in the analysis of machine learning methods. Moreover, starting with the work \cite{ADPZ11}, it has been established that there is a close link between gradient flows on the space of probability measures and large deviations for particle systems. In a sense, an LDP identifies the most natural gradient flow formulation for a given PDE; see \cite{ADPZ13, DPZ13, PGN21} and references therein. With the increased use of particle systems and gradient flows in the analysis of machine learning methods, it is natural to introduce the large deviation framework in this setting. Our work can be seen as a first step in this direction, by establishing a relevant LDP in the setting of zero-sum games. This opens up the possibility for more detailed analysis of computational methods for finding mixed Nash equilibria, including for the training of GANs, based on the associated rate function. 

For future work, we are interested in extending the results to variants of mirror-descent-like algorithms, as appearing in, e.g., \cite{HLC19, DEJ+20, BKPP21}, and computing the corresponding rate functions for specific examples of minmax games appearing in the training of GANs. The latter will be used to investigate the impact different parameters have on convergence and to compare different methods (i.e., particle dynamics). A stochastic version of what in \cite{DEJ+20} is referred to as Wasserstein-Fisher-Rao dynamics is of particular interest.

\subsection*{Notation}
\label{sec:notation}
For a space $\calS$, we take $\calP (\calS)$ to denote the space of probability measures on $\calS$ and $C([0,T]: \calS)$ to denote the space of continuous functions from $[0,T]$ into $\calS$. We take $C_b (\calS)$ to be the space of real-valued, bounded functions on $\calS$. If $\calS$ is a metric space, with metric $d:\calS \times \calS \to \bR$, we take $BL(\calS)$ to be the set of functions that are bounded and Lipschitz on $\calS$, with norm
\[
    \norm{f}_{BL} = \norm{f}_\infty + \Lip f,
\]
where $\Lip f$ is the Lipschitz constant of $f$. The dual bounded-Lipschitz metric on $\calP (\calS)$ is defined as 
\[
	d_{BL} (\mu, \nu) = \sup _{f \in BL(\calS), \norm{f} _{BL} \leq 1 }  \int _{\calS} f d\left[\mu - \nu \right], 
\]
and metrises the topology of weak convergence on $\calP(\calS)$. Another family of metrics on $\calP (\calS)$ that will be used is the Wasserstein metrics: for $p \geq 1$, the $p$-Wasserstein metric on the space of probability measures on $\calS$ with finite $p$th moment defined as
\[
	W_p (\mu, \nu) = \left[ \inf _{\gamma \in \Gamma (\mu, \nu)} \int _{\calS \times \calS} d(x,y) ^p d \gamma (x,y) \right] ^{1/p},
\]
where $\Gamma (\mu, \nu)$ is the collection of all couplings of $\mu$ and $\nu$. Convergence in $W_p$ is equivalent to weak convergence plus convergence of the first $p$ moments; see \cite{Vil09} for a more detailed discussion of the Wasserstein metrics.

Lastly, for an element $s$ of a product space $\calS _1 \times \calS _2$, we use $s_1$ to denote the $\calS _1$-component of $s$ and $s_2$ to denote the $\calS _2$-component.

\section{Definitions and model setup}
The setting of interest in this paper is two-player zero-sum games and methods for finding mixed Nash equilibria for such games. To define the relevant concepts and quantities, let $\calX =\bR ^{d_x}$ and $\calY = \bR ^{d_y}$, for some $d_x, d_y \geq 1$; for convenience we set $\calZ = \calX \times \calY$. We focus on the Euclidean setup here, although generalisation to Riemannian manifolds, the setting used in \cite{DEJ+20}, can also be considered. A \textit{two-player zero-sum game} consists of a set of two players with parameters $x \in \calX$ and $y \in \calY$ and a function $l: \calZ \to \bR$ that gives the payoff for player one. That is, for $(x,y) \in \calZ$, the payoff for player one, the $x$-player, is $l(x,y)$ and the payoff of player two, the $y$-player, is $-l(x,y)$. To ease comparison with \cite{DEJ+20} we refer to $l$ also as the \textit{loss} of the game.

A \textit{pure Nash equilibrium point} (NE) \cite{Nash51} is a pair of strategies $(x^*, y^*) \in \calZ$ such that 
\begin{align}
\label{eq:NE}
	l(x^*, y) \leq l(x^*, y^*) \leq l(x, y^*), \ \ \forall (x,y) \in \calZ.
\end{align}
A \textit{mixed Nash equilibrium point} (MNE) \cite{Nash50, Gli52} is a pair of probability distributions $(\mu _x ^*, \mu _y ^*) \in \calP (\calX) \times \calP (\calY)$ such that, for all $\mu_x \in \calP (\calX)$, $\mu_y \in \calP (\calY)$ ,
\[
	\int _\calY \int _\calX l(x,y) d \mu _x ^* (x) d\mu _y (y) \leq \int _\calY \int _\calX l(x,y) d \mu _x ^* (x) d\mu _y ^* (y) \leq \int _\calY \int _\calX l(x,y) d \mu _x (x) d\mu _y ^* (y).
\]
Similar to \cite{DEJ+20}, to simplify the notation we take $L$ to denote the expected loss: For $\mu, \nu$ in $\calP (\calX)$ and $\calP (\calY)$,
\[
	L (\mu, \nu) = \int _\calY \int _\calX l(x,y) d\mu (x) d\nu (y).
\]
The definition of an MNE then becomes: for all $\mu_x \in \calP (\calX), \mu_y \in \calP (\calY)$,
\[
	L(\mu _x ^*, \mu _y) \leq L (\mu _x ^*, \mu _y ^*) \leq L (\mu _x, \mu _y ^*).
\]
Comparing this to \eqref{eq:NE} shows that an MNE for the loss $l$ is a pure Nash equilibrium point with respect to $L$. For a given loss function $l$, neither Nash equilibria nor MNEs are guaranteed to exist for continuous games. However, MNEs exist in greater generality (see, e.g. \cite{Gli52}) and are more viable for computational methods. 

The aim of \cite{HLC19, DEJ+20} is to construct efficient computational methods for finding approximations of an MNE for a given loss function $l$. A common way to quantify the accuracy of such an approximation $(\hat \mu _x, \hat \mu _y )$, used also in \cite{DEJ+20}, is the \textit{Nikaid\^o and Isoda (NI) error} \cite{NI55}: 
\begin{align}
\label{eq:NI}
	NI(\hat \mu _x, \hat \mu _y) = \sup _{\mu _y \in \calP (\calY)} L (\hat \mu _x, \mu _y) - \inf _{\mu _x \in \calP (\calX)} L (\mu _x, \hat \mu _y).
\end{align}
A pair of distributions $(\tilde \mu _x, \tilde \mu _y) \in \calP(\calX) \times \calP (\calY)$ is called an \textit{$\ve$-MNE} if it satisfies $NI(\tilde \mu_x , \tilde \mu _y) \leq \ve$. Note that for an MNE $(\mu _x ^*, \mu _y ^*)$, $NI(\mu _x ^*, \mu ^* _y) = 0$.

As mentioned in Section \ref{sec:intro}, although existence is less of an issue when we consider MNEs instead of pure NEs, finding an MNE for a given game is typically a difficult task. In \cite{DEJ+20} the authors propose an approach for approximating MNEs associated with $l$ based on a mean-field dynamics. Whereas the mean-field dynamics are defined in terms of a PDE, actual computations are based on the following interacting particle system: for $i =1, \dots, n$, let $W^i$ and $\bar W ^i$ be independent Wiener processes on $\calX$ and $\calY$, respectively, take $\beta > 0$ and for some initial distributions $\mu _{x,0}$ and $\mu _{y,0}$, consider the system of $2n$ coupled SDEs
\begin{align}
	\label{eq:IPS}
	\begin{split}
	dX^i _t &= - \frac{1}{n} \sum _{j=1} ^n \grad _x l (X^i_t, Y^j _t ) dt + \sqrt{2 \beta ^{-1}} dW^i _t, \ \ X^i _0 = \xi ^i \sim \mu_{x,0}, \\
	dY^i _t &= \frac{1}{n} \sum _{j=1} ^n \grad_y l (X^j_t, Y^i _t ) dt + \sqrt{2 \beta ^{-1}} d \bar W^i _t, \ \ Y^i _0 =\bar \xi ^i \sim \mu_{y,0}.
	\end{split}
\end{align}
In \cite{DEJ+20} this is referred to as the Langevin ascent-descent dynamics, or entropic regularisation (Algorithm 1 therein). The dynamics also resembles those of Algorithm 4 in \cite{HLC19}, therein referred to as approximate infinite-dimensional mirror descent. 

 In \cite{DEJ+20} the authors go on to prove a series of results for the limit $n \to \infty$ and about the Wasserstein gradient flow associated with the corresponding mean-field dynamics. First, they show that the empirical measure of the particle system \eqref{eq:IPS} converges, with respect to the $W_2$-metric and uniformly in time, to a solution of the gradient flow \eqref{eq:iwgf-dominguez}, 
  \begin{equation}\label{eq:iwgf-dominguez}
\begin{cases}
    \partial_{t} \mu_{x}=\nabla_{x} \cdot\left(\mu_{x} \nabla_{x} V_{x}\left(\mu_{y}, x\right)\right)+\beta^{-1} \Delta_{x} \mu_{x}, \\ 
    \partial_{t} \mu_{y}=-\nabla_{y} \cdot\left(\mu_{y} \nabla_{y} V_{y}\left(\mu_{x}, y\right)\right)+\beta^{-1} \Delta_{y} \mu_{y}, 
\end{cases}
\end{equation}
where
\begin{align*}
    V_x(\mu _y, x) = \int_{\calY} l(x,y) d \mu_y (y), \ \  V_y (\mu_x, y) = \int _{\calX} l(x,y) d \mu_x (x).
\end{align*}
It is also shown that the mean absolute error of the $NI$-error associated with the empirical measure converges to 0 (uniformly in time), that if the solution of \eqref{eq:iwgf-dominguez} is to converge in time, then the limit can be characterised as a certain fixed point, and is an $\varepsilon$-MNE for $L$ given that $\beta$ is above a specified threshold. 
 
In the following section we establish large deviation results for the empirical measures of the particle system \eqref{eq:IPS} and show that some of the convergence results in \cite{DEJ+20} become corollaries of the corresponding LDP. 

\section{An LDP and a.s.-convergence of the NI error}
Consider the coupled systems of SDEs \eqref{eq:IPS}. We view this as describing the dynamics of $2n$ particles over the (arbitrary) time interval $[0,T]$. The empirical measures for the two collections of particles are defined as
\begin{align}
\label{eq:EmpMeas}
	\mu ^n (\cdot) = \frac{1}{n} \sum_{i=1} ^n \delta _{ X^i} (\cdot), \ \ \nu ^n = \frac{1}{n} \sum _{i=1} ^d \delta _{Y^i} (\cdot).
\end{align}
These measures are viewed as elements in $\calP (C([0,T]: \calX))$ and $\calP (C([0,T]:\calY))$, respectively, where we equip $C([0,T]: \calX)$ and $C([0,T]:\calY)$ with the supremum norm. For $t \in [0,T]$, the corresponding $t$-marginals are denoted $\mu ^n _t (\cdot)= \frac{1}{n} \sum_{i=1} ^n \delta _{ X^i _t} (\cdot)$ and $\nu ^n _t (\cdot)= \frac{1}{n} \sum_{i=1} ^n \delta _{ Y^i _t} (\cdot)$, which belong to $\calP (\calX)$ and $\calP(\calY)$, respectively. Henceforth we make the following assumption on the loss $l$.
\begin{assumption}
\label{ass:loss}
The loss function $l$ is continuous and bounded, and $\grad l$ is Lipschitz and bounded. 

\end{assumption}
Under this assumption, the regularity properties of $l$ carries over to the coefficients appearing in an alternative form of \eqref{eq:IPS}, ensuring that these are regular enough for existing large deviation results to apply. It is clear that the assumptions can be weakened in different directions, in particular the assumption that $l$ and $\grad l$ are bounded. However, the focus of this note is no to obtain the most general results possible, why we are content with using this more convenient assumption for now. Note also that in \cite{DEJ+20} the underlying parameter spaces are assumed to be compact, which under their assumptions on $l$ implicitly ensures boundedness of both $l$ and $\grad l$.

In order to state the main large deviation result, we first introduce some notation from \cite{BDF12}, in which the LDP is established on an augmented space (see \cite{BDF12} for a more thorough description). Let $\calS = C([0,T]: \calZ)$, the trajectory space for each particle $Z^i = (X^i, Y^i)$, and $\calW = C([0,T]: \calZ)$, the trajectory space for the Wiener process, both equipped with the maximum norm. Let $\calR$ denote the space of deterministic relaxed controls on $\calZ \times [0,T]$ and $\calR _1$ the corresponding subset of elements with finite first moment. $\calR _1$ is equipped with the topology generated by the $W_1$-metric applied to normalised versions of the elements in $\calR _1$: For any $r, \tilde r \in \calR _1$, consider the distance $d_{\calR _1} (r, \tilde r) = W_1 (r/T, \tilde r / T)$, where $r/T$ is the measure $\calS \times [0,T] \ni A \mapsto r(A)/T$ (see \cite{BDF12} and references therein for further details related to these spaces).	

Within the space $\calP (\calS \times \calR_1 \times \calW)$, let $\calP _{\infty} $ denote the collection of measures $\Theta \in \calP (\calS \times \calR _1 \times \calW)$ that satisfy
\begin{itemize}
	\item $\int _{\calR _1} \int _{\calZ \times [0,T]} ||z ||^2 r(dz dt) \Theta _{\calR _1} (dr) < \infty$,
	\item $\Theta$ is a weak solution of the SDE
	\begin{align}
	\label{eq:limitSDE}
	\begin{split}
		d \tilde X (t) &= b (\tilde X (t), \calL (\tilde X (t)))dt + \sqrt{2 \beta ^{-1}} \int _{\calZ} z \rho _t (dz) dt + \sqrt{2 \beta ^{-1}} d \hat W (t), \\
		\tilde X (0) &\sim \gamma _0,
		\end{split} 
	\end{align}
	where $\hat W$ is a standard Wiener process defined on some probability space and under the corresponding probability measure the triplet $(\tilde X, \rho, \hat W)$ has distribution $\Theta$; $\calL (\tilde X (t))$ is the law of the random variable $\tilde X (t)$.
\end{itemize}
The main result in \cite{BDF12} states that the sequence $\{ \gamma ^n \}$ of empirical measures satisfies the LDP with rate function (see Remark 3.2 in \cite{BDF12})

\begin{align}
	\label{eq:rate3}
	J (\theta) = \inf _{\Theta \in \calP _\infty : \Theta _{\calS} = \theta} \bE _{\Theta} \left[ \frac{1}{2} \int _0 ^T \norm{u(t)} ^2 dt \right], \ \ \theta \in \calP (C([0,T]:\calZ)),
\end{align}
where
\begin{align*}
	u(t) = \int _{\calZ} z \rho _t (dz), 
\end{align*}
and the triple $(\bar Z, \rho, \bar W)$ is the canonical process on $\calS \times \calR _1 \times \calW$ (equipped with its Borel $\sigma$-algebra) and $\Theta$-a.s.\ $\bar Z$ satisfies
\begin{align}
\label{eq:barZ}
	d \bar Z (t) = b (\bar Z (t), \theta (t)) dt + \sqrt{2 \beta ^{-1}}u(t)dt + \sqrt{2 \beta ^{-1}} d \bar W (t).
\end{align}
 This is a control-type formulation of the rate function, where $u$ acts as the control in the limit equation for $\bar Z$. 

We are now ready to state the main large deviation result. The key observation is that, under Assumption \ref{ass:loss}, the LDP for the pair of empirical measures $(\mu ^n, \nu ^n)$ can be shown using the results from \cite{BDF12} described in the previous paragraphs. 

\begin{theorem}
\label{thm:LDP}
For each $n \in \bN$, define $\gamma ^n =  \frac{1}{n} \sum _{i=1} ^n \delta _{(X^{i,n}, Y^{i,n})} \in \calP (C([0,T]:\calZ))$. Assume that $\gamma ^n _0 \to \gamma _0$, for some $\gamma _0 \in \calP (\calZ)$, as $n \to \infty$. Then, the family of empirical measures $\{ \gamma ^n \} _{n \in \bN}$ satisfies an LDP, on $\calP (C([0,T]:\calZ))$, with speed $n$ and rate function $J$ given in \eqref{eq:rate3}.

Moreover, under the same conditions and with $\gamma ^n$ viewed as an element of $C([0,T]: \calP (\calZ))$, $\{ \gamma ^n \} _{n \in \bN}$ satisfies an LDP with speed $n$ and rate function
	\begin{align}
	\label{eq:rate2}
		I(\theta) = \frac{\beta}{4} \int _0 ^T \sup _{f: \langle |\nabla_z f|^2, \theta (t) \rangle \neq 0} \frac{ \langle \dot \theta (t) - A_\theta ^* \theta (t) , f \rangle ^2}{\langle | \grad _z f |^2, \theta (t) \rangle} dt,
	\end{align}
	where the supremum is taken over real Schwartz functions $f$ on $\calZ$ and, for an element $\theta \in C([0,T]:\calP(\calZ))$, $A^* _{\theta}$ is defined in \eqref{eq:PDEZ}.
\end{theorem}
Before we give a proof, we comment briefly on the result. First, those familiar with large deviation theory recognise this as a Dawson-G\"artner type result. Whereas the original results by Dawson and G\"artner \cite{DG87} can be applied in the current setting, with different assumptions on $l$ and $\grad l$, their results require using a certain inductive topology on the space of probability measures; see, e.g., \cite{BDF12, Fis14} for more on this. The results in \cite{DG87} are also restricted to diffusion coefficients that are non-degenerate and independent of the empirical measures $\mu ^n, \nu ^n$, whereas for future work we are interested in dynamics where the latter does not hold. We have therefore opted to use a more flexible approach to the large deviation result also in this simpler setting. Lastly, we expect that the alternative formulation \eqref{eq:rate2} of the rate function will be beneficial for studying the performance of the proposed algorithms, similar to how a parellell form of the rate function for small-noise diffusions has been used in the context of importance sampling \cite{DSZ15, BD19}

\begin{proof}[Proof of Theorem \ref{thm:LDP}]
	The proof follows from the LDP in \cite{BDF12} once we express the particle dynamics in an appropriate way and verify that the assumptions of \cite{BDF12} hold. For the first part, we define a two-component particle system, with components $X^n$ and $Y^n$: set $Z^n = (Z^{1,n}, \dots, Z^{n,n})$, where $Z^{i,n} = (X^{i,n}, Y^{i,n})$. With this definition, we identify $\gamma ^n$ as the empirical measure of $Z^n$:
	\begin{align}
	\label{eq:gamma_n}
	\gamma ^n (\cdot) = \frac{1}{n} \sum _{i=1} ^n \delta _{(X^{i,n}, Y^{i,n})} (\cdot)  = \frac{1}{n} \sum _{i=1} ^n \delta _{Z^{i,n}} (\cdot) \ \in \calP (C([0,T]:\calX \times \calY)).
	\end{align}
The empirical measures $\mu ^n$ and $\nu ^n$, and their $t$-marginals, are now the marginals of $\gamma ^n$ and $\gamma ^n _t$, e.g., $\mu _t ^n (\cdot) = \gamma_t ^n (\cdot \times \calY)$. To ease the notation we set $\gamma ^n _\calX (\cdot) = \gamma ^n (\cdot \times \calY) $, $\gamma ^n _{\calX,t}$ to be the corresponding $t$-marginal and analogously for $\gamma ^n _\calY$, $\gamma ^n _{\calY, t}$.
	
	With the dynamics \eqref{eq:IPS} for $X^n$ and $Y^n$, the dynamics for the ``new'' system $Z^n$ can be expressed as, for $i=1, \dots, n$,
	\begin{align}
	\label{eq:Zn}
	\begin{split}
		dZ^{i,n} _1 (t) &= - \int _{\calX \times \calY} \grad _x l (Z^{i,n} _1 (t), y) d\gamma ^n _t (x,y) dt + \sqrt{2 \beta ^{-1}}dW ^i (t),\\
		dZ ^{i,n} _2 (t) &= \int _{\calX \times \calY} \grad _y l (x, Z^{i,n} _2 (t)) d\gamma ^n _t (x,y) dt + \sqrt{2 \beta ^{-1}}d\bar W ^i (t).
	\end{split}
	\end{align}

	To express this in a more standard form, similar to the interacting particle systems treated in \cite{BDF12}, we define the function $b: \calZ \times \calP (\calZ) \to \calZ$ as $b (z, \eta) = \left( b_1 (z, \eta), b _2 (z, \eta) \right)$, with
	\begin{align}
	\label{eq:defb}
		\begin{split}
			b_1 (z, \eta) &= - \int _\calY \grad _x l(z_1, y) d\eta (x, y), \\
			b_2 (z, \eta) &= \int _{\calX} \grad_y l(x, z_2) d\eta (x,y).
		\end{split}
	\end{align}
	If we also define the process $\tilde W ^i$ on $\calX \times \calY$ as $\tilde W ^i (t) = (W^i (t), \bar W ^i (t))$, a Wiener process on $\calZ$, then the dynamics for the $Z^{i,n}$s can be expressed as follows: for $t \in [0,T]$, $Z^{i,n} (t)$ is a solution of the SDE
	\begin{align}
	\label{eq:DefZ}
		dZ ^{i,n} (t) = b (Z^{i,n} (t) , \gamma ^n _t) dt + \sqrt{2 \beta ^{-1}} d\tilde W ^i (t).
	\end{align}
	This SDE is of the form considered in \cite{BDF12}. It remains to check that the assumptions used therein for proving the LDP also hold in our setting. As mentioned in \cite{BDF12}, it suffices to have the drift $ b$ be uniformly Lipschitz\footnote{Weaker conditions than this, and the one imposed in this paper, can suffice as well; see comment before the proof and \cite{BDF12} for a more extensive discussion.}. We now prove that this holds under Assumption \ref{ass:loss}. 
	
Take $C > 0$ to be such that $\sup _{z \in \calZ} ||\grad l (z) || \leq C$ and let $L$ be the Lipschitz constant of $\grad l$. Then, $\grad _x l / (C \vee L)$ and $\grad _y l / (C\vee L)$ are both in $\Lip \calZ$. For any $\mu, \nu \in \calP(\calZ)$ and $z, z' \in \calZ$, we have for $b_1$ that
	\begin{align}
	\label{eq:ineqb1}
	\left| \left|  b _1 (z,\mu) - b_1 (z', \nu) \right| \right| \leq \left| \left| b_1 (z, \mu) -  b _1 (z, \nu)  \right| \right| + \left| \left|  b_1 (z', \nu) - b_1 (z, \nu) \right| \right|.	
	\end{align}  
	For the first term on the right-hand side of \eqref{eq:ineqb1}, we have 
	\begin{align*}
	&\left| \left|  b_1 (z, \mu) -  b _1 (z, \nu)  \right| \right| \\
	&= \left| \left| \int _{\calX \times \calY} \grad_x l (z_1, y) d\mu(x, y) - \int _{\calX \times \calY} \grad_x l (z_1, y) d\nu(x, y) \right| \right| \\
	&=  (C \vee L) \left| \left| \int _{\calX \times \calY} \frac{\grad_x l (z_1, y)}{C \vee L} d\mu(x, y) - \int _{\calX \times \calY}  \frac{\grad_x l (z_1, y)}{C \vee L} d\nu(x, y) \right| \right| \\
	&\leq (C \vee L) d_{BL} (\mu, \nu).
	\end{align*}
	Considering now the second term on the right-hand side of \eqref{eq:ineqb1}, we use the Lipschitz property of $\grad l$:
	\begin{align*}
		\left| \left|  b_1 (z', \nu) -  b_1 (z, \nu) \right| \right| &= \left| \left|  \int _{\calX \times \calY} \left( \grad _x l (z_1, y) - \grad_x l (z' _1, y) \right) d\nu(x,y) \right| \right| \\
		&\leq  \int _{\calX \times \calY} \left| \left| \grad _x l (z_1, y) - \grad_x l (z' _1, y) \right| \right| d\nu(x,y)  \\
		&\leq L \int _{\calX \times \calY} ||z_1 - z_1 ' || d \nu (x,y) \\
		& = L ||z_1 - z' _1||.
	\end{align*}
	The calculations for $ b _2$ are completely analogous and we conclude that 
	\begin{align*}
	\left| \left| b  (z,\mu) - b (z', \nu) \right| \right| \leq& \left| \left| b _1 (z,\mu) - b_1 (z', \nu) \right| \right| + \left| \left| b _2 (z,\mu) - b_2 (z', \nu) \right| \right| \\
	\leq & (C \vee L) d_{BL} (\mu, \nu) + L ||z_1 - z' _1|| \\
	& \quad + (C \vee L) d_{BL} (\mu, \nu) + L || z_2 - z' _2|| \\
	= & 2 (C \vee L) d_{BL} (\mu, \nu) + L \left( ||z_1 - z_1 '|| + ||z_2 - z_2 '|| \right).
	\end{align*}
	This shows that $b$ is Lipschitz.
	
	Because the diffusion coefficient $\sqrt{2 \beta ^{-1}}$ is constant, the global Lipschitz property of $ b$ is enough for an application of \cite[Theorem 3.1]{BDF12}. 
 
 It remains to move from the rate function \eqref{eq:rate3} to the calculus of variations form \eqref{eq:rate2}. Starting with the control-formulation, we can use the contraction principle to obtain an LDP for $\{ \gamma ^n\}$ on $C([0,T]: \calP (\calZ))$. Let $I$ denote the corresponding rate function. Standard calculations suggest that it can be re-written as \eqref{eq:rate2}. With $A ^u _{\theta}$ the generator associated with \eqref{eq:barZ}, and $(A ^u _{\theta}) ^*$ the corresponding (formal) adjoint---with $A^* _\theta$ corresponding to the case $u\equiv 0$---the PDE characterisation of the dynamics \eqref{eq:barZ} is
\begin{align}
\label{eq:PDEZ}
	\frac{d}{dt} \theta(t) &=  (A ^u _{\theta}) ^* \theta (t) \notag \\
		&= - \grad _z \left( \theta (t) b(\cdot, \theta (t)) \right) + \sqrt{2 \beta ^{-1}} \grad _z \left( u(t) \theta (t) \right) + \beta ^{-1} \Delta _z \theta (t),
\end{align}
which is interpreted in the weak sense. 
We now take $\calD$ to be the Schwartz space of real distributions on $\calZ$ and, for $t\in [0,T]$ and $\mu \in \calP (\calZ)$, define the norm $\norm{\cdot}_{\mu, t} $ by
\begin{align*}
	\norm{\eta}_{\mu, t}^2 = \sup _{f } \frac{\langle f, \eta \rangle ^2}{2 \beta ^{-1} \langle |\grad _z f|^2, \mu \rangle }, \ \ \eta \in \calD,
\end{align*}
where the supremum is taken over real Schwartz test functions $f$ on $\calZ$ such that $\langle |\grad _z f|^2, \mu \rangle \neq 0$. Based on the characterisation \eqref{eq:PDEZ}, with $u \equiv 0$, we have the representation 
\begin{align*}
	I (\theta) &= \frac{1}{2} \int _0 ^T \norm{ \dot \theta (t) - A_{\theta} ^* \theta (t)} ^2 _{\theta (t), t} dt \\
	& = \frac{\beta}{4} \int _0 ^T \sup _{f: \langle |\grad _z f|^2, \theta (t) \rangle \neq 0} \frac{ \langle \dot \theta (t) - A_\theta ^* \theta (t) , f \rangle ^2}{\langle | \grad _z f |^2, \theta (t) \rangle} dt.
\end{align*}
See, e.g., \cite{DG87, BDF12} for a heuristic description of how to go from \eqref{eq:PDEZ} and the control formulation \eqref{eq:rate3} to this version of the rate function. In the recent works \cite{BZ20, BZ22} the authors give, to the best of our knowledge, the first rigorous proof of this type of equivalence between the different types of rate functions, for both moderate and large deviations, in the setting of multi-scale processes. The above expression for $I$ is precisely \eqref{eq:rate2}, the prescribed form of the rate function.
\end{proof}

Theorem \ref{thm:LDP} establishes the relevant LDP by appealing to results in \cite{BDF12}. A first by-product is a law-of-large-numbers-type convergence of the empirical measures $\gamma ^n$ defined in \eqref{eq:gamma_n}. 
\begin{corollary}
\label{cor:MFlim}
The sequence of empirical measures $\gamma ^n$ converges, as $n \to \infty$, almost surely in $C([0,T]: \calP(\calZ))$ to $\gamma$, the solution of \eqref{eq:PDEZ} with $u \equiv 0$:
\begin{align}
\label{eq:gamma}
	\dot \gamma (t)= \grad _z \left( \gamma (t) b (\cdot, \gamma (t))\right) + \beta ^{-1} \Delta _z \gamma (t).
\end{align}
\end{corollary}
\begin{proof}
	The rate function $J$ is minimised for the control $u \equiv 0$, which corresponds to the original (uncontrolled) dynamics for one pair of particles, described by \eqref{eq:barZ} with $u = 0$. The law of this process is unique (see \cite{DEJ+20}) and we have that as $n \to \infty$, $\gamma ^n$ converges to this law a.s. An application of the continuous mapping theorem then gives the desired result, recognising that the $u\equiv 0$-dynamics correspond to \eqref{eq:PDEZ} with the same choice for $u$.
\end{proof}

Corollary \ref{cor:MFlim} is a version of the first part of Theorem 3 in \cite{DEJ+20}. To see this, insert the definition \eqref{eq:defb} of $b$ into \eqref{eq:gamma}: the resulting equation is precisely the entropy-regularised gradient flow \eqref{eq:iwgf-dominguez}. More generally, the LDP for a particle systems identifies a natural candidate for the gradient flow structure of the limit as $n \to \infty$: both the dissipation mechanism and entropy functional (see \cite{AGS08}) can be identified from the LDP, see \cite{ADPZ11, ADPZ13} and subsequent work by Peletier and co-authors. In \cite{DEJ+20} the gradient flow is used to propose the dynamics used for finding MNEs. For other particle dynamics, aimed at the same task, this type of large deviation analysis can be used to identify the correct gradient flow to use for further analysis. This will be the subject of future work on mirror-descent-like algorithms.

In \cite{DEJ+20}, to establish the convergence to the gradient flow, the authors work in the $W_2$-topology on $\calP (\calX)$ and $\calP (\calY)$, consider the convergence of the $t$-marginals for $t \in [0,T]$ and conclude that the convergence is uniform in $t$. Here, the result follows from the LDP in Theorem \ref{thm:LDP}, by noting that $\gamma$ as in \eqref{eq:gamma} satisfies $I(\gamma) = 0$. A standard argument using the Borel-Cantelli lemma then shows that $\gamma$ is indeed the limit of $\gamma ^n$ as $n \to \infty$. 
 

The next result, which also follows from Theorem \ref{thm:LDP}, corresponds to the second part of Theorem 3 in \cite{DEJ+20}, modulo the different topologies and mode of convergence being used. In order to have results that are uniform over $[0,T]$, and ease notation, we define the map, for $\eta \in C[0,T]:\calP(\calZ))$,
\begin{align*}
	NI_T (\eta) = \{ NI( \eta _{\calX, t}, \eta _{\calY, t} ) \}_{t \in [0,T]}.
\end{align*}

\begin{proposition}
\label{prop:NIconv}
The sequence $ \{ NI_T (\gamma ^n ) \} _n$ converges almost surely in $C([0,T]:\bR)$ to $NI_T (\gamma )$, where $\gamma $ is the same as in Corollary \ref{cor:MFlim}. 
\end{proposition}
\begin{proof}

As a first step we establish that the map $\eta _t \mapsto NI (\eta _{\calX, t}, \eta _{\calY, t})$ is continuous for each $t \in [0,T]$. In Lemma 2 in \cite{DEJ+20} it is shown that the NI error, defined in \eqref{eq:NI}, is a Lipschitz map on $\calP(\calX) \times \calP(\calY)$ when the $W_1$ distance is used to define the topology on $\calP (\calX) \times \calP (\calY)$. Therein the underlying state spaces are also assumed compact. Removing compactness, Assumption \ref{ass:loss} is enough for the main ideas of the proof to be used also for the topology of weak convergence on $\calP(\calX)$ and $\calP (\calY)$. First, similar to \cite{DEJ+20}, we note that under Assumption \ref{ass:loss}, the function $x \mapsto \int l(x,y) d \mu _y (y)$, for any $\mu_y \in \calP (\calY)$, is continuous, bounded and Lipschitz. 

With this observation, we now adapt the arguments from Lemma 2 in \cite{DEJ+20} to establish the continuity of the NI error. Instead of $W_1$, we work with the dual bounded-Lipschitz metric, and rather than the Lipschitz constant of $l$ we use an upper bound $C$ on the bounded Lipschitz norm of the function $x \mapsto \int l(x,y) d \mu _y (y)$. We note that this upper bound $C$ can be taken independent of $\mu _y$ due to the assumptions on $l$. Using this property, for any two $\mu_x, \tilde \mu _x \in \calP (\calX)$, the steps used in \cite[Lemma 2]{DEJ+20} leads to the upper bound 

\begin{align*}
	\left| \sup _{\mu _y \in \calP (\calY)} L(\mu_x, \mu _y) - \sup _{\mu _y \in \calP (\calY)} L(\tilde \mu_x, \mu _y) \right| \leq C d_{BL} (\mu _x, \tilde \mu _x).
\end{align*}
Using the analogous argument for the map $y \mapsto \int l(x,y) d \mu _x (x)$, for any $\mu _x \in \calP (\calX)$, we have
\begin{align*}
	\left| \sup _{\mu _x \in \calP (\calX)} L(\mu_x, \mu _y) - \sup _{\mu _x \in \calP (\calX)} L(\mu_x, \tilde \mu _y) \right| \leq C d_{BL} (\mu _y, \tilde \mu _y),
\end{align*}
for any $\mu _y, \tilde \mu _y \in \calP (\calY)$. The continuity of the NI error now follows from an application of the triangle inequality.

With the desired continuity established, we can now show the claimed convergence of $NI_T(\gamma ^n)$. From Corollary \ref{cor:MFlim} we have convergence of the marginals of $\gamma ^n$ to those of the solution $\gamma$ of \eqref{eq:gamma}. Combined with the continuity obtained in the previous paragraphs, for each $t \in [0,T]$, with probability one we have
\begin{align*}
	NI (\gamma ^n _{\calX, t}, \gamma ^n _{\calY, t}) \to NI (\gamma _{\calX, t}, \gamma _{\calY, t}), \ \ n \to \infty.
\end{align*}

To obtain the convergence uniformly in $t$, we employ again the argument used to show that the NI error is a continuous function on $\calP (\calX) \times \calP (\calY)$. Repeating the steps outlined above, similar to those in \cite{DEJ+20}, we arrive at, for any $t \in [0,T]$, 
\begin{align*}
    |NI (\gamma ^n _t) - NI(\gamma _t)| &\leq \tilde C d_{BL}(\gamma ^n _t, \gamma _t),
\end{align*}
a.s., where $\tilde C$ is an upper bound on the $BL$-norm of  $x \mapsto \int l(x,y) d \mu _y (y)$ (for any $\mu _y$) and $y \mapsto \int l(x,y) d \mu _x (x)$ (for any $\mu _x$). By the convergence $\gamma ^n \to \gamma$, for a given $\tilde \ve >0$, there is a $N_{\tilde \ve} \in \bN$ such that $d_{BL}(\gamma ^n , \gamma) \leq \tilde \ve$ a.s.\ for all $n \geq N _{\tilde \ve}$. This also implies the same property for the $t$-marginals: a.s., for any $t \in [0,T]$, $d_{BL} (\gamma ^n _t, \gamma _t) \leq \tilde \ve$ if $n \geq N_{\tilde \ve}$. Therefore, since the constant $\tilde C$ does not depend on $n$ or $t$,
\begin{align*}
    \sup _{t\in [0,T]} |NI (\gamma ^n _t) - NI(\gamma _t)| &\leq \tilde C \sup _{t \in [0,T]} d_{BL}(\gamma ^n _t, \gamma _t) \\
    & \leq \tilde C \tilde \ve,
\end{align*}
a.s.\ for all $n \geq N _{\tilde \ve}$. For a given $\ve > 0$, pick $\tilde \ve = \ve / \tilde C$, so that the inequality becomes
\[
    \sup _{t\in [0,T]} |NI (\gamma ^n _t) - NI(\gamma _t)| \leq \ve,
\]
a.s.\ for all $n \geq N_{\tilde \ve}$. This shows the claimed convergence.
\end{proof}

In addition to the convergence in Proposition \ref{prop:NIconv}, as a consequence of Theorem \ref{thm:LDP} we also obtain the corresponding LDP for the NI error. 
\begin{corollary}
\label{cor:LDPNI}
	 The sequence $\{NI_T (\gamma ^n) \} _n$ satisfies an LDP on $C([0,T]:\bR)$, with speed $n$ and rate function, for $\alpha \in C([0,T]:\bR)$,
	 \begin{align}
	 	\label{eq:rateNI}
		\tilde I (\alpha) = \inf \left\{ I(\theta); \theta \in C([0,T]:\calP(\calZ)) \ \textrm{s.t. }  \  NI_T(\theta) = \alpha \right\}.
	 \end{align} 
\end{corollary}
The proof is a direct consequence of combining Theorem \ref{thm:LDP}, the continuity of the map $\eta \mapsto NI_T (\eta)$ and the contraction principle \cite{Dembo98}. Note that an LDP for the NI error at a fixed time $t \in [0,T]$ follows by applying the contraction principle once more.

\subsection*{Acknowledgments}
The authors are grateful to Z.~W.~Bezemek and K.~Spiliopoulos for useful discussions of a first version of the paper. In particular regarding the relation between the different formulations of the rate functions in Theorem \ref{thm:LDP}. 

The research of VN and PN was supported by Wallenberg AI, Autonomous Systems and Software Program (WASP) funded by the Knut and Alice Wallenberg Foundation. The research of PN was also supported by the Swedish Research Council (VR-2018-07050).


\end{document}